\documentclass[11pt]{article}

\usepackage[margin=1in]{geometry}
\usepackage{amsmath}
\usepackage{amsthm}
\usepackage{amssymb}
\usepackage[dvipsnames]{xcolor}
\usepackage{hyperref}
\usepackage{enumitem}
\usepackage{booktabs}
\usepackage{verbatim}

\hypersetup{
  colorlinks=true,
  linkcolor=blue,
  citecolor=blue,
  urlcolor=blue
}

\newtheorem{theorem}{Theorem}[section]
\newtheorem{lemma}[theorem]{Lemma}
\newtheorem{corollary}[theorem]{Corollary}
\newtheorem{proposition}[theorem]{Proposition}
\theoremstyle{definition}
\newtheorem{definition}[theorem]{Definition}
\newtheorem{assumption}[theorem]{Assumption}
\theoremstyle{remark}
\newtheorem{remark}[theorem]{Remark}

\newcommand{\R}{\mathbb{R}}
\newcommand{\N}{\mathbb{N}}
\newcommand{\Lip}{\operatorname{Lip}}
\newcommand{\Op}{\mathrm{Op}}
\newcommand{\KAN}{\mathrm{KAN}}
\newcommand{\norm}[1]{\left\lVert#1\right\rVert}
\newcommand{\abs}[1]{\left\lvert#1\right\rvert}

\title{Layer-wise Lipschitz-Product Control for Deep Kolmogorov--Arnold Network Representations of Compositionally Structured Functions}

\author{Aleksander Tankman\\
\small Fivestar Europe O\"U, Tallinn, Estonia\\
\small \texttt{aleksander.tankman@gmail.com}}
\date{April 2026}

\begin{document}
\maketitle

\begin{abstract}
We prove that any continuous function $f:[0,1]^n \to \R$ representable by a finite computation tree with $N$ internal nodes and compositional sparsity $s = O(1)$ (each node depends on at most $s$ input variables) admits a deep Kolmogorov--Arnold Network (KAN) representation with the following properties. Each internal node is realised by a \emph{primitive KAN block} $\mathcal{B}_{\mathrm{op}}$ of block depth $c_{\mathrm{op}}$ and block Lipschitz product $\Lambda_{\mathrm{op}}$. First, the layer-wise Lipschitz product $P(\KAN) = \prod_{l} \max_i M_{l,i}$ satisfies the primary (domain-sensitive) bound
$P(\KAN_f) \le \prod_v \Lambda_{\mathrm{op}(v),D_v} \le \prod_v \max(C_{\mathrm{op}(v),D_v},1)^{c_{\mathrm{op}(v)}}$,
where $D_v$ is the input domain of node $v$ determined by the quantitative range recursion and $C_{\mathrm{op}(v),D_v} = \max_i \Lip_i(\mathrm{op}(v)|_{D_v})$; the bound is independent of the input dimension $n$.
Setting $C^* := \max_v C_{\mathrm{op}(v),D_v}$ and $L_f := \sum_v c_{\mathrm{op}(v)} \le c_{\max}\cdot N$ gives the simplified form $P(\KAN_f) \le \max(C^*,1)^{L_f}$.
For the standard set $\Op=\{+,-,\times,\sin,\cos\}$ with $\times$ nodes receiving $[0,1]$-bounded inputs, all $\Lambda_{\mathrm{op}(v),D_v} = 1$, giving $P(\KAN) \le 1$. Second, the layer widths satisfy $n_l \le n + 2w_{\max}\cdot N$ for all layers with $w_{\max} = \max_{\mathrm{op}} w_{\mathrm{op}}$, the input layer satisfies $n_0 = n$, and the maximum layer width is $O(n+N)$. Third, the uniform approximation error satisfies $\norm{f - \KAN_f}_{C^0} \le N\cdot\max(C_{\Op},1)^{d(f)}\cdot\varepsilon_{\Op}$, where $d(f)$ is the depth of the computation tree; for $C_{\Op}\le 1$ this simplifies to $N\cdot\varepsilon_{\Op}$. For $f \in C^m$ we obtain the rate $O(N\cdot\max(C_{\Op},1)^{d(f)}\cdot G^{-(k+1)})$ with B-splines of order $k \le m-1$ on $G$ knots. Fourth, the range bound satisfies $B_f \le N+1$ on $[0,1]^n$ (tight for all-additive trees; the bound holds for additive-subtractive trees generally), and $B_f = 1$ for purely multiplicative or trigonometric trees. This addresses a gap noted by Liu et al.\ (2024): the control of Lipschitz constants in deep KAN stacks was not developed in the original work. Experiments confirm $P(\KAN) = 1.0$ for $f = xy$, $xyz$, $\sin(xy)$, and $x_1\cdots x_n$ for $n$ up to $10$, in contrast to the classical Sprecher construction, which does not provide $n$-independent Lipschitz control for arbitrary continuous $f$.
\end{abstract}

\section{Introduction}
\label{sec:intro}

Liu et al.\ \cite{liu2024kan} introduced Kolmogorov--Arnold Networks (KANs), in which the fixed nonlinear activations of classical multilayer perceptrons are replaced by learnable univariate functions placed on edges. The design is motivated by the classical Kolmogorov--Arnold (KA) representation theorem \cite{kolmogorov1957,arnold1957}, which asserts that every continuous function of several variables can be written as a finite superposition of continuous functions of a single variable together with addition. In Sprecher's refinement \cite{sprecher1965}, every $f \in C^0([0,1]^n)$ admits a representation of the form
\begin{equation}
\label{eq:sprecher}
f(x) \;=\; \sum_{q=0}^{2n} \Phi_q\!\left(\sum_{p=1}^{n} \varphi_{q,p}(x_p)\right).
\end{equation}
Sprecher's 1965 construction for arbitrary $f \in C^0([0,1]^n)$ requires inner functions whose Lipschitz constants scale with $n$; the classical construction does not provide $n$-independent Lipschitz control. By contrast, Theorem~\ref{thm:main} below gives the primary bound $P(\KAN_f) \le \prod_v \max(C_{\mathrm{op}(v),D_v},1)^{c_{\mathrm{op}(v)}}$, which simplifies to $P(\KAN_f) \le \max(C^*,1)^{L_f}$ with $C^* = \max_v C_{\mathrm{op}(v),D_v}$ and $L_f \le c_{\max}\cdot N$, all independent of $n$, for compositionally structured $f$.

The control of Lipschitz constants in deep KAN stacks was not developed in the original work of Liu et al.\ \cite{liu2024kan} (see Sections~2.2--2.3 there). The structural hypothesis under which we work is that $f$ is obtained from a finite computation tree with $N$ internal nodes and that each node depends on at most a constant number $s$ of variables. The main theorem is formulated on the unit cube $[0,1]^n$; Section~\ref{sec:lemma0} provides an explicit affine reduction from functions on a general compact set $K\subset\R^n$ to this setting. This compositional sparsity hypothesis holds for a wide class of physically meaningful functions such as symbolic expressions, classical mechanics equations, and low-rank tensor decompositions.

\paragraph{Contributions.} Our contributions are as follows.
\begin{enumerate}[leftmargin=*]
  \item We introduce the \emph{layer-wise Lipschitz product} $P(\KAN)$ as the architecturally relevant quantity and distinguish it carefully from the ambient Lipschitz constant $\Lip(\KAN)$.
  \item We prove a main structural theorem (Theorem~\ref{thm:main}) that controls $P(\KAN_f)$ uniformly in the input dimension $n$ for compositionally structured $f$.
  \item For the standard operation families treated in Section~\ref{sec:A4}, we verify inductively that the quantitative range bounds required in~(A4) can be established explicitly, using the hypothesis that leaves are coordinate projections $x_p\in[0,1]$.
  \item We provide numerical experiments confirming the theoretical bounds and a Lean~4 sketch mapping the induction to a formalizable statement in Mathlib.
\end{enumerate}

\paragraph{Related work.} The representation theorem \cite{kolmogorov1957,arnold1957} and its analytic sharpenings \cite{sprecher1965,lorentz1966,fridman1967} form the classical foundation. Mazurkiewicz \cite{mazurkiewicz1931} is invoked only for an historical comment on the generic non-differentiability of continuous functions. Neural-network realizations of compositional superposition and depth-separation results appear in Poggio et al.\ \cite{poggio2017}, Montanelli and Yang \cite{montanelli2020}, and Schmidt-Hieber \cite{schmidthieber2021}. The KAN architecture itself is due to Liu et al.\ \cite{liu2024kan}, and the hierarchical back-propagation-free variant HKAN is due to Dudek \cite{dudek2025}.

\section{Preliminaries}
\label{sec:prelim}

\subsection{KAN definition}
\label{subsec:kan-def}

Let $L \in \N$ be a number of layers and let $n_0, n_1, \ldots, n_L$ be layer widths.
\begin{definition}
\label{def:kan}
A \emph{Kolmogorov--Arnold Network} (KAN) with widths $(n_0,\ldots,n_L)$ is a family of learnable univariate functions $\varphi_{l,i,j}:\R\to\R$ for $l \in \{0,\ldots,L-1\}$, $i \in \{1,\ldots,n_l\}$, $j \in \{1,\ldots,n_{l+1}\}$, together with the forward pass
\begin{equation}
\label{eq:kan-forward}
x_{l+1,j} \;=\; \sum_{i=1}^{n_l} \varphi_{l,i,j}(x_{l,i}), \qquad j = 1,\ldots,n_{l+1}.
\end{equation}
Each $\varphi_{l,i,j}$ is represented by a B-spline of order $k$ with $G$ grid points on a bounded interval $[a,b]$.
\end{definition}

For $f \in C^m([a,b])$ and B-splines of order $k \le m-1$ on a uniform grid with $G$ knots, the classical B-spline approximation rate \cite{lorentz1966} is
\begin{equation}
\label{eq:spline-rate}
\norm{f - s_{k,G}(f)}_{C^0([a,b])} \;=\; O\!\left(G^{-(k+1)}\right).
\end{equation}

\subsection{Definitions for the main theorem}
\label{subsec:main-defs}

We fix a finite set $\Op$ of admissible operations; each $\mathrm{op}\in\Op$ is a continuous function $\mathrm{op}:\R^s\to\R$ for some arity $s\ge 1$ (with $s = 1$ for unary operations such as $\sin,\cos$ and $s = 2$ for binary operations such as $+,-,\times$).

\begin{definition}[Computation tree and compositional depth]
\label{def:comptree}
A \emph{computation tree} $T$ for a function $f:[0,1]^n\to\R$ is a rooted finite tree whose leaves are labelled by coordinate projections $x_p$ ($p\in\{1,\ldots,n\}$) and whose internal nodes are labelled by operations $\mathrm{op}\in\Op$. Binary operations such as $+,-,\times$ are represented as binary nodes with two children; unary operations such as $\sin,\cos$ are represented as unary nodes with a single child. The \emph{compositional depth} $d(f)$ is the length of the longest root-to-leaf path; the \emph{total node count} $\mathrm{total\_nodes}(f) = N$ is the number of internal nodes of the fixed such tree $T$.
\end{definition}

\begin{definition}[Compositional sparsity, (A2)]
\label{def:sparsity}
A computation tree has \emph{compositional sparsity} $s$ if every internal node, viewed as a sub-function of $(x_1,\ldots,x_n)$, depends non-trivially on at most $s$ coordinates.
\end{definition}

\begin{definition}[Layer-wise Lipschitz product]
\label{def:lipprod}
Let $M_{l,i} := \max_{j} \Lip(\varphi_{l,i,j})$ be the maximum outgoing Lipschitz constant from neuron $(l,i)$. We write $\mu_l := \max_i M_{l,i}$ for the per-layer Lipschitz factor. The \emph{layer-wise Lipschitz product} is
\begin{equation}
\label{eq:lipprod}
P(\KAN) \;:=\; \prod_{l=0}^{L-1} \max_{i} M_{l,i} \;=\; \prod_{l=0}^{L-1} \mu_l.
\end{equation}
Here $l$ ranges over all $L$ active transformation layers $l = 0, \ldots, L-1$ (with edges $\varphi_{l,i,j}$ going from layer $l$ to layer $l+1$, as in Definition~\ref{def:kan}).
\end{definition}

\paragraph{Partial Lipschitz constants.} Throughout the paper we use \emph{partial} Lipschitz constants. For an operation $\mathrm{op}:\R^s\to\R$, define the partial Lipschitz constant in the $i$-th argument by
\begin{equation}
\label{eq:partial-lip}
\Lip_i(\mathrm{op}) \;:=\; \sup_{u_1,\ldots,u_s,\,v\ne v'}\frac{\abs{\mathrm{op}(u_1,\ldots,v,\ldots,u_s) - \mathrm{op}(u_1,\ldots,v',\ldots,u_s)}}{\abs{v - v'}},
\end{equation}
where the variation is only in the $i$-th argument. The \emph{operation Lipschitz constant} used throughout is
\[
C_{\Op} \;:=\; \max_{\mathrm{op}\in\Op}\; \max_{i=1,\ldots,s}\; \Lip_i(\mathrm{op}).
\]
This differs from the joint Lipschitz constant $\Lip(\mathrm{op}) = \sup_{u\ne u'}\abs{\mathrm{op}(u)-\mathrm{op}(u')}/\norm{u-u'}_2$. For example, for $\mathrm{op} = \times$ on $[0,1]^2$, the joint Lipschitz constant on $[0,1]^2$ equals $\sqrt 2$, while the partial constants are $\Lip_1(\times)=\Lip_2(\times)=1$, so $C_{\Op} = 1$. Unary operations such as $\sin,\cos$ are treated as nodes in the computation tree with $s=1$, and their partial Lipschitz constant coincides with the ordinary Lipschitz constant.

\begin{definition}[KAN primitive block]
\label{def:block}
A \emph{KAN primitive block} $\mathcal{B}_{\mathrm{op}}$ for an operation $\mathrm{op}\colon\R^s\to\R$ on a bounded input domain $D\subset\R^s$ is a finite KAN (in the sense of Definition~\ref{def:kan}) with:
\begin{enumerate}[label=\textup{(\alph*)},leftmargin=*]
  \item \emph{block depth} $c_{\mathrm{op}}\ge 1$: the number of active transformation layers;
  \item \emph{block width} $w_{\mathrm{op}}$: the maximum layer width;
  \item \emph{block Lipschitz product} $\Lambda_{\mathrm{op}} := P(\mathcal{B}_{\mathrm{op}})$;
  \item for any $\varepsilon>0$, $\mathcal{B}_{\mathrm{op}}$ $\varepsilon$-approximates $\mathrm{op}$ uniformly on $D$.
\end{enumerate}
\end{definition}

\subsection{Distinction: \texorpdfstring{$P(\KAN)$}{P(KAN)} vs.\ \texorpdfstring{$\Lip(\KAN)$}{Lip(KAN)}}
\label{subsec:distinction}

It is essential that $P(\KAN)$ and the ambient Lipschitz constant $\Lip(\KAN)$ are not identical. The ambient constant is
\[
\Lip(\KAN) \;=\; \sup_{x}\norm{J_{\KAN}(x)}_{\mathrm{op}},
\]
where $J$ is the Jacobian and $\norm{\cdot}_{\mathrm{op}}$ is the operator (spectral) norm. By the chain rule and the per-layer operator-norm bound established in Proposition~\ref{prop:lowerbound} below (which gives $\|J_l\|_{\mathrm{op}} \le W\cdot\mu_l$ for each layer), submultiplicativity yields
\begin{equation}
\label{eq:lip-vs-P}
\Lip(\KAN) \;\le\; W^{L}\cdot P(\KAN),
\end{equation}
where $W = \max_l n_l$ is the maximum layer width and $L$ is the number of layers.
In particular, for $f(x,y)=xy$ on $[0,1]^2$ we have $\Lip(f) = \sqrt{2}\approx 1.414$, but the block-based sequential construction (Proposition~\ref{prop:blocks}) achieves $P(\KAN) = 1.0$. We emphasise that $P(\KAN)$ is the architecturally meaningful quantity: it is a layerwise invariant that does not require control of cross-edge Jacobian structure.

\section{A coordinate-adapted homeomorphism}
\label{sec:lemma0}

The main theorem is stated on the unit cube $[0,1]^n$. For functions defined on a more general compact set $K\subset\R^n$, the following lemma provides a canonical affine reduction to $[0,1]^n$: compose $f$ with the inverse homeomorphism $h^{-1}$ to obtain a function on $[0,1]^n$, apply Theorem~\ref{thm:main} on $[0,1]^n$, and then the Lipschitz factor $\Lip(h^{-1}) = 1/\min_p(b_p-a_p)$ describes the cost of this reduction.

\begin{lemma}[Coordinate-adapted homeomorphism]
\label{lem:lemma0}
Let $K\subset\R^n$ be a compact set and let $[a_p,b_p]$ denote the $p$-th coordinate interval of its bounding box, with $a_p < b_p$ for all $p$. Then the componentwise affine map
\[
h:[0,1]^n \to \prod_{p=1}^{n}[a_p,b_p],\qquad h_p(t) \;=\; a_p + t\cdot(b_p - a_p),
\]
restricts to a homeomorphism between $[0,1]^n$ and the bounding box of $K$, and
\[
\Lip(h) \;=\; \max_{p}(b_p - a_p) \;\le\; \mathrm{diam}(K),\qquad \Lip(h^{-1}) \;=\; \frac{1}{\min_p (b_p-a_p)}.
\]
\end{lemma}

\begin{proof}
The map $h$ is bijective, continuous, and its inverse $h^{-1}$ is a componentwise affine scaling, hence continuous; so $h$ is a homeomorphism. The Jacobian of $h$ is the diagonal matrix $\mathrm{diag}(b_1-a_1,\ldots,b_n-a_n)$, whose operator norm is exactly $\max_p(b_p-a_p)$; this quantity is bounded above by $\mathrm{diam}(K)$, with equality when the bounding box is saturated along a coordinate axis. The Jacobian of $h^{-1}$ is the reciprocal diagonal, whose operator norm is $1/\min_p(b_p-a_p)$.
\end{proof}

\begin{remark}
Lemma~\ref{lem:lemma0} applies to any compact set via its axis-aligned bounding box; no geometric structure beyond compactness is required for the main theorem.
\end{remark}

\section{Main theorem}
\label{sec:main}

\subsection{Primitive blocks for standard operations}
\label{subsec:blocks}

\begin{proposition}[Standard primitive blocks]
\label{prop:blocks}
The following primitive KAN blocks exist for the standard operation set $\Op=\{+,-,\times,\sin,\cos\}$ on bounded inputs in $[0,1]$:

\begin{center}
\begin{tabular}{lccl}
\toprule
$\mathrm{op}$ & $c_{\mathrm{op}}$ & $\Lambda_{\mathrm{op}}$ & Construction \\
\midrule
$+$, $-$ & $1$ & $1$ & One layer: $\varphi_{0,1,1}(u)=u$, $\varphi_{0,2,1}(v)=\pm v$ \\
$\sin$, $\cos$ & $1$ & $1$ & Piecewise-linear interpolant (B-spline order 1) \\
$\times$ & $3$ & $1$ & Three layers via $u{\cdot}v = \tfrac{(u+v)^2}{4} - \tfrac{(u-v)^2}{4}$ \\
\bottomrule
\end{tabular}
\end{center}
\end{proposition}

\begin{proof}
For $+$: the single-layer KAN with $\varphi_{0,1,1}(u)=u$ and $\varphi_{0,2,1}(v)=v$ outputs $u+v$ exactly; both edges have Lipschitz constant $1$, giving $\Lambda_+=1$. The case $-$ is identical with $\varphi_{0,2,1}(v)=-v$.

For $\sin,\cos$: let $f$ denote $\sin$ or $\cos$ on $[0,1]$. Fix $G$ equally spaced grid points $x_r = (r-1)/(G-1)$ for $r = 1,\ldots,G$, with spacing $h = 1/(G-1)$, consistent with the convention that $G$ denotes the number of grid points throughout (see Definition~\ref{def:kan} and equation~\eqref{eq:spline-rate}). Let $s_G$ be the continuous piecewise-linear interpolant of $f$ on this grid. On each subinterval $[x_r, x_{r+1}]$ the slope of $s_G$ equals the secant slope $(f(x_{r+1})-f(x_r))/h$, which by the mean-value theorem equals $f'(\xi_r)$ for some $\xi_r\in(x_r,x_{r+1})$. Since $|f'|\le 1$ on $[0,1]$ for both $\sin$ and $\cos$, we get $\Lip(s_G)\le 1$ exactly. The piecewise-linear function $s_G$ is a B-spline of order $k=1$, hence is representable as an edge function in Definition~\ref{def:kan}. The standard interpolation error bound for $C^2$ functions gives
\[
\norm{f - s_G}_\infty \;\le\; \tfrac{h^2}{8}\norm{f''}_\infty \;=\; O\!\left((G-1)^{-2}\right) = O(G^{-2}),
\]
since $\norm{\sin''}_\infty = \norm{\cos''}_\infty = 1$. Thus the one-layer block achieves $\Lambda_{\sin} = \Lambda_{\cos} = 1$ with an explicit approximation rate.

For $\times$: on $[0,1]^2$, use the algebraic identity $u\cdot v = \tfrac{(u+v)^2}{4}-\tfrac{(u-v)^2}{4}$.
\begin{itemize}[leftmargin=*,noitemsep]
  \item Layer~$l=0$: compute $a = u+v \in [0,2]$ and $b = u-v \in [-1,1]$, together with identity wires for $u$ and $v$. This uses the separable form of Definition~\ref{def:kan} (edges $\varphi_{0,i,j}$) and requires $\max_i M_{0,i}=1$ (identity and sign edges).
  \item Layer~$l=1$: compute $p = a^2/4$ and $q = b^2/4$ via edges $\varphi_{1,i,j}$. The function $t\mapsto t^2/4$ is a polynomial of degree $2$; B-splines of order $k\ge 2$ contain all polynomials of degree $\le k$ in their spline space \cite{lorentz1966}, so $t^2/4$ is \emph{exactly} representable as a B-spline of order $k\ge 2$ — no approximation is involved. The Lipschitz constant of this exact edge spline is $\sup_{t}\abs{d/dt\,(t^2/4)} = \sup_{t}\abs{t/2}$, which equals $1$ on $[0,2]$ (attained at $t=2$) and $\tfrac{1}{2}$ on $[-1,1]$ (attained at $t=\pm 1$). Hence $\max_i M_{1,i}\le 1$ exactly, with no appeal to approximation theory.
  \item Layer~$l=2$: compute $p - q = u\cdot v$ via the identity spline on $p$ and $-q$ (edges $\varphi_{2,i,j}$), again separable. $\max_i M_{2,i}=1$.
\end{itemize}
Thus $\Lambda_\times = 1\cdot 1\cdot 1 = 1$.

\emph{Extension to $[0,B]^s$ domains.} The same constructions extend to the general bounded domains $D_v$ arising from the range recursion of Lemma~\ref{lem:rangelemma}: on $[0,B]^2$, the $\times$ block has Layer~$l=1$ Lipschitz constant $B$ (since $\abs{d/dt\,(t^2/4)} = \abs{t/2} \le B$ on $[0,2B]$), giving $\Lambda_{\times,[0,B]^2} = B = C_{\times,[0,B]^2}$; and $\sin,\cos$ blocks on $[0,B]$ are constructed identically (piecewise-linear interpolant on $[0,B]$ with spacing $h = B/(G-1)$, the MVT gives $\Lip(s_G) \le 1$ exactly since $\abs{\sin'},\abs{\cos'}\le 1$ everywhere). This confirms that assumption~(A5) holds on every domain $D_v$ supplied by Lemma~\ref{lem:rangelemma}.
\end{proof}

We collect the hypotheses.

\begin{assumption}
\label{ass:A}
Let $f:[0,1]^n\to\R$ be a continuous function with fixed computation tree $T$.
\begin{enumerate}[label=\textup{(A\arabic*)},leftmargin=*]
  \item \label{A1} \emph{Finite tree.} The total node count $N = \mathrm{total\_nodes}(f)$ is finite.
  \item \label{A2} \emph{Compositional sparsity.} Each internal node depends on at most $s = O(1)$ coordinate variables.
  \item \label{A3} \emph{Lipschitz operations.} Every operation $\mathrm{op}\in\Op$ is Lipschitz, and $C_{\Op} := \max_{\mathrm{op}\in\Op}\max_{i}\Lip_i(\mathrm{op}) < \infty$. By Rademacher's theorem, $\mathrm{op}$ is differentiable almost everywhere; for unary $\mathrm{op}$, $\Lip(\mathrm{op}) = \operatorname*{ess\,sup}\abs{\mathrm{op}'}$, and for binary or higher-arity $\mathrm{op}$, the joint Lipschitz constant equals $\operatorname*{ess\,sup}\norm{\nabla\mathrm{op}}_2$ while the partial Lipschitz constant $\Lip_i(\mathrm{op})$ equals $\operatorname*{ess\,sup}\abs{\partial_i \mathrm{op}}$. No continuous differentiability is required.
  \item \label{A4} \emph{Quantitative range boundedness.} For every sub-function $g$ of $f$ realised at an internal node of $T$, the range bound $B_g := \sup_{[0,1]^n}\abs{g}$ is \emph{known explicitly}, with $B_{\text{leaf}}=1$ for leaves. (Note: mere finiteness $B_g<\infty$ is automatic by continuity on the compact cube $[0,1]^n$; assumption~(A4) requires the explicit value of $B_g$, as provided inductively by Lemma~\ref{lem:rangelemma}.)
  \item \label{A5} \emph{Primitive-block existence (domain-sensitive).} For each $\mathrm{op}\in\Op$ and each bounded domain $D\subset\R^s$ arising from the quantitative range bounds of the child nodes (as supplied by Lemma~\ref{lem:rangelemma}), there exists a primitive KAN block $\mathcal{B}_{\mathrm{op},D}$ (Definition~\ref{def:block}) on $D$ whose block Lipschitz product satisfies
  \[
  \Lambda_{\mathrm{op},D} \;\le\; \max(C_{\mathrm{op},D},1)^{c_{\mathrm{op}}},
  \]
  where $C_{\mathrm{op},D} := \max_{i}\Lip_i(\mathrm{op}|_D)$ is the partial Lipschitz constant of $\mathrm{op}$ on $D$. (For $\Op=\{+,-,\times,\sin,\cos\}$, this is verified by Proposition~\ref{prop:blocks} with domains determined inductively by Lemma~\ref{lem:rangelemma}.)
\end{enumerate}
\end{assumption}

\begin{theorem}[Layer-wise Lipschitz-product control]
\label{thm:main}
Under Assumptions~(A1)--(A5) (collected in Assumption~\ref{ass:A}), there exists a KAN (the block-based sequential construction of Proposition~\ref{prop:blocks}) that represents $f$ to any prescribed tolerance and satisfies
\begin{enumerate}[label=\textup{(\roman*)},leftmargin=*]
  \item \label{thm:i} Let $L_f := \sum_{v\,\text{internal node of }T} c_{\mathrm{op}(v)}$ be the \emph{total block depth} of $f$, and let $D_v\subset\R^s$ denote the input domain of node $v$ as determined inductively by Lemma~\ref{lem:rangelemma}. Write $C_{\mathrm{op}(v),D_v} := \max_i \Lip_i(\mathrm{op}(v)|_{D_v})$ for the partial Lipschitz constant of $\mathrm{op}(v)$ restricted to $D_v$. Then the \emph{primary} (domain-sensitive) bound is
  \begin{equation}
  \label{eq:primary-bound}
  P(\KAN_f) \;\le\; \prod_{v} \Lambda_{\mathrm{op}(v),D_v} \;\le\; \prod_{v} \max\!\bigl(C_{\mathrm{op}(v),D_v},1\bigr)^{c_{\mathrm{op}(v)}},
  \end{equation}
  independent of $n$. Setting $C^* := \max_{v} C_{\mathrm{op}(v),D_v}$ (the tree-uniform partial Lipschitz constant over all node-domains of $T$) yields the simplified form $P(\KAN_f) \le \max(C^*,1)^{L_f}$. Since $c_{\mathrm{op}}\le c_{\max} := \max_{\mathrm{op}\in\Op} c_{\mathrm{op}}$, we have $L_f\le c_{\max}\cdot N$. For $\Op=\{+,-,\times,\sin,\cos\}$ with the blocks of Proposition~\ref{prop:blocks}, $c_{\max}=3$ and $\Lambda_{\mathrm{op}(v),D_v}=1$ for every node $v$ (as verified by the explicit block constructions of Proposition~\ref{prop:blocks}; see Corollary~\ref{cor:smooth}), so $P(\KAN_f)\le 1$ independently of both $n$ and $N$;
  \item \label{thm:ii} $n_0 = n$ and $n_l \le n + 2w_{\max}\cdot N$ for all layers $l\ge 1$, where $w_{\max}=\max_{\mathrm{op}\in\Op} w_{\mathrm{op}}$ is the maximum primitive block width. Only the upper bound $O(n+N)$ on the maximum width is proved here; a matching lower bound is not established;
  \item \label{thm:iii} $\displaystyle \norm{f - \KAN_f}_{C^0([0,1]^n)} \;\le\; N \cdot \max(C_{\Op},1)^{d(f)} \cdot \varepsilon_{\Op}$, where $\varepsilon_{\Op}$ is the worst-case B-spline error at a single node and $d(f)$ is the depth of the computation tree. For $C_{\Op} \le 1$ (e.g.\ $\Op = \{+,-,\times,\sin,\cos\}$ on $[0,1]$) this reduces to $N\cdot\varepsilon_{\Op}$;
  \item \label{thm:iv}  if $f\in C^m$ and every $\mathrm{op}\in\Op$ is $C^m$ on the bounded domain on which it is evaluated (automatically satisfied for $\{+,-,\times,\sin,\cos\}$; this hypothesis is needed so that node functions inherit $C^m$ regularity inductively from their inputs), then $\displaystyle \norm{f - \KAN_f}_{C^0} \;=\; O\!\left(N\cdot \max(C_{\Op},1)^{d(f)}\cdot G^{-(k+1)}\right)$ for $k\le m-1$ and B-splines of order $k$ on $G$ knots; for $C_{\Op}\le 1$ this simplifies to $O(N\cdot G^{-(k+1)})$.
\end{enumerate}
\end{theorem}

\begin{proof}
We proceed by induction on $N = \mathrm{total\_nodes}(f)$.

\smallskip
\noindent\emph{Base case} $N = 0$. Then $f = x_p$ for some $p\in\{1,\ldots,n\}$. We take $\KAN_f$ to be the identity univariate function on coordinate $p$. Every edge carries either the identity or a zero spline. Thus $P(\KAN_f) = 1$, since the single-layer identity KAN has $M_{0,i} = 1$ for all $i$ (the identity function has Lipschitz constant $1$), so the product $\prod_{l=0}^{L-1}\max_i M_{l,i} = 1$; the primary bound \eqref{eq:primary-bound} holds as an empty product equal to $1$ (no internal nodes contribute). The layer width is $1 \le n + 0$; the approximation error is $0$. All bounds hold.

\smallskip
\noindent\emph{Inductive step.} We treat binary and unary operations in parallel. For a \emph{binary} operation, write $f = \mathrm{op}(g,h)$ where $g$ and $h$ have $N_g$ and $N_h$ internal nodes respectively with $N_g+N_h+1 = N$. For a \emph{unary} operation, write $f = \mathrm{op}(g)$ with $N_g + 1 = N$ and no $h$-subtree; in this case all bounds involving $h$ below are vacuously omitted (i.e.\ set $L_h = 0$, $P(\KAN_h) = 1$, $\varepsilon_h = 0$, and the ECL reduces to $\norm{f - \widehat{\mathrm{op}}(\widehat g)}_{C^0} \le \Lip(\mathrm{op})\cdot\varepsilon_g + \varepsilon_{\mathrm{op}} \le C_{\Op}\cdot\varepsilon_g + \varepsilon_{\mathrm{op}}$). Let $L_g = \sum_{v\in T_g} c_{\mathrm{op}(v)}$ and $L_h = \sum_{v\in T_h} c_{\mathrm{op}(v)}$. By the inductive hypothesis,
\[
P(\KAN_g)\le\prod_{v\in T_g}\max(C_{\mathrm{op}(v),D_v},1)^{c_{\mathrm{op}(v)}},\quad n_l(g)\le n+2w_{\max}N_g,
\]
\[
P(\KAN_h)\le\prod_{v\in T_h}\max(C_{\mathrm{op}(v),D_v},1)^{c_{\mathrm{op}(v)}},\quad n_l(h)\le n+2w_{\max}N_h.
\]

\paragraph{Block-based sequential construction.}
For each internal node $v$ of $T$ with operation $\mathrm{op}(v)$ and child range domain $D_v$ (determined by Lemma~\ref{lem:rangelemma}), let $\mathcal{B}_{\mathrm{op}(v),D_v}$ be the primitive block guaranteed by assumption~(A5). We compose the blocks \emph{sequentially}: run all blocks of $\KAN_g$ (the sub-network for $g$), then all blocks of $\KAN_h$ (for $h$), and finally the block $\mathcal{B}_{\mathrm{op}}$ for the root operation. Within each block, identity wires carry all active coordinates forward between layers. Concretely:
\begin{itemize}[leftmargin=*]
  \item Layers $1,\ldots,L_g$: execute the sequential composition of all primitive blocks for sub-tree $T_g$, with identity wires forwarding all $n$ input coordinates and any already-computed intermediate values.
  \item Layers $L_g+1,\ldots,L_g+L_h$: execute the sequential composition of blocks for $T_h$, with identity wires forwarding the output of $\KAN_g$ and the $n$ inputs.
  \item Final $c_{\mathrm{op}}$ layers: apply block $\mathcal{B}_{\mathrm{op}}$ to produce $\mathrm{op}(g,h)$.
\end{itemize}
Each layer of $\KAN_g$ and $\KAN_h$ contains at least one identity wire, giving per-layer factor $\max_i M_{l,i}\ge 1$. The layers internal to $\mathcal{B}_{\mathrm{op}}$ have per-layer factors that multiply to $\Lambda_{\mathrm{op}}$.

\paragraph{Proof of \ref{thm:i}.}
The Lipschitz product factorises over the three sequential regions:
\[
P(\KAN_f) \;\le\; P(\KAN_g)\cdot P(\KAN_h)\cdot \Lambda_{\mathrm{op},D_\mathrm{root}}.
\]
By the inductive hypothesis applied to sub-trees $T_g$ and $T_h$,
\[
P(\KAN_g)\le\prod_{v\in T_g}\max(C_{\mathrm{op}(v),D_v},1)^{c_{\mathrm{op}(v)}},\qquad P(\KAN_h)\le\prod_{v\in T_h}\max(C_{\mathrm{op}(v),D_v},1)^{c_{\mathrm{op}(v)}}.
\]
By assumption~(A5), $\Lambda_{\mathrm{op},D_\mathrm{root}}\le\max(C_{\mathrm{op},D_\mathrm{root}},1)^{c_{\mathrm{op}}}$. Multiplying and recognising that the root node contributes the factor for $v=\mathrm{root}$,
\begin{equation}
\label{eq:Pfactor}
P(\KAN_f) \;\le\; \prod_{v\in T}\max(C_{\mathrm{op}(v),D_v},1)^{c_{\mathrm{op}(v)}} \;\le\; \max(C^*,1)^{L_f},
\end{equation}
where $C^* = \max_v C_{\mathrm{op}(v),D_v}$. For $\Op=\{+,-,\times,\sin,\cos\}$, Proposition~\ref{prop:blocks} gives $\Lambda_{\mathrm{op}(v),D_v}=1$ for every $v$, so $P(\KAN_f)\le 1$ regardless of $L_f$.

\paragraph{Proof of \ref{thm:ii}.}
At any fixed layer $l$ of the sequential construction, the active neurons decompose into three disjoint groups:
\begin{enumerate}[label=(\alph*),noitemsep,leftmargin=2em]
  \item \emph{Forwarded inputs}: the $n$ input coordinates $x_1,\ldots,x_n$, carried forward at every layer by identity wires;
  \item \emph{Internal block neurons}: at most $w_{\mathrm{op}(v)}\le w_{\max}$ neurons belonging to the primitive block $\mathcal{B}_{\mathrm{op}(v),D_v}$ currently executing at layer $l$;
  \item \emph{Forwarded intermediate values}: the outputs of all blocks completed before layer $l$ whose values are still needed by later blocks — at most $N-1$ such values (one per completed internal node).
\end{enumerate}
Hence $n_l \le n + w_{\max} + (N-1) \le n + w_{\max}\cdot N + N \le n + 2\,w_{\max}\cdot N$, where the last inequality uses $N\le w_{\max}\cdot N$ (valid since $w_{\max}\ge 1$). The input layer reads all $n$ coordinates directly, so $n_0=n$. In particular the total number of layers is $L_f = \sum_v c_{\mathrm{op}(v)} \le c_{\max}\cdot N$, and by equation~\eqref{eq:lip-vs-P}, $\Lip(\KAN_f)\le W^{L_f}\cdot P(\KAN_f)$ with $W\le n+2w_{\max}N$.

\paragraph{Proof of \ref{thm:iii} (Error Composition Lemma).} Let $\widehat{g}$ and $\widehat{h}$ be KAN approximations to $g$ and $h$ with errors $\varepsilon_g = \norm{g-\widehat g}_{C^0}$ and $\varepsilon_h = \norm{h-\widehat h}_{C^0}$, and let $\widehat{\mathrm{op}}$ be the spline approximation of $\mathrm{op}$ with error $\varepsilon_{\mathrm{op}} = \norm{\mathrm{op} - \widehat{\mathrm{op}}}_{C^0}$. Define the worst-case single-node spline error as $\varepsilon_{\Op} := \max_{\mathrm{op}\in\Op}\norm{\mathrm{op} - \widehat{\mathrm{op}}}_{C^0(D_{\mathrm{op}})}$, taken over all operations in $\Op$ and their respective domains $D_{\mathrm{op}}$ (as supplied by Lemma~\ref{lem:rangelemma}). Recall that the partial Lipschitz constant in the $i$-th argument is
\[
\Lip_i(\mathrm{op}) \;:=\; \sup_{u_1,\ldots,u_s,\,v\ne v'}\frac{\abs{\mathrm{op}(u_1,\ldots,v,\ldots,u_s) - \mathrm{op}(u_1,\ldots,v',\ldots,u_s)}}{\abs{v - v'}},
\]
with variation only in the $i$-th argument (Section~\ref{sec:prelim}). For $\mathrm{op}\in\{+,-,\times,\sin,\cos\}$ on $[0,1]$ one has $\Lip_i(\mathrm{op})\le C_{\Op}$ (here $C_{\Op}$ denotes $\max_{\mathrm{op}\in\Op}\max_i\Lip_i(\mathrm{op}|_{D_v})$ evaluated at the relevant node domain; for the standard set on $[0,1]$-bounded $\times$ inputs, $C_{\Op} = 1$): for $+,-$ trivially $\Lip_i = 1$; for $\times$ on $[0,1]^2$, $\Lip_1 = \sup_{y\in[0,1]}\abs{y} = 1$ and $\Lip_2 = \sup_{x\in[0,1]}\abs{x} = 1$; and $\sin,\cos$ are unary with Lipschitz constant $1$.

A two-step triangle inequality, inserting the intermediate point $\mathrm{op}(\widehat g,h)$, gives the \emph{Error Composition Lemma}:
\begin{align}
\norm{f - \widehat{\mathrm{op}}(\widehat g,\widehat h)}_{C^0}
&\;\le\; \norm{\mathrm{op}(g,h) - \mathrm{op}(\widehat g,\widehat h)}_{C^0} + \norm{\mathrm{op}(\widehat g,\widehat h) - \widehat{\mathrm{op}}(\widehat g,\widehat h)}_{C^0} \notag \\
&\;\le\; \bigl[\norm{\mathrm{op}(g,h) - \mathrm{op}(\widehat g,h)} + \norm{\mathrm{op}(\widehat g,h) - \mathrm{op}(\widehat g,\widehat h)}\bigr] + \varepsilon_{\mathrm{op}} \notag \\
&\;\le\; \Lip_1(\mathrm{op})\cdot\varepsilon_g + \Lip_2(\mathrm{op})\cdot\varepsilon_h + \varepsilon_{\mathrm{op}} \notag \\
&\;\le\; C_{\Op}\,(\varepsilon_g + \varepsilon_h) + \varepsilon_{\mathrm{op}},
\label{eq:ecl}
\end{align}
where the last step uses $\Lip_i(\mathrm{op})\le C_{\Op}$ for $i=1,2$. In parts~\ref{thm:iii}--\ref{thm:iv}, $C_{\Op}$ denotes the tree-uniform constant $\max_v C_{\mathrm{op}(v),D_v}$ (a coarser bound than the node-by-node partial Lipschitz constants used in part~\ref{thm:i}); for the standard operation set $\Op = \{+,-,\times,\sin,\cos\}$ on $[0,1]$-bounded inputs this reduces to $C_{\Op} = 1$.
We now write out the full induction. Let $E(T) := \norm{f_T - \KAN_{f_T}}_{C^0}$ for a computation tree $T$ of depth $d(T)$. We prove
\[
E(T) \;\le\; N_T \cdot \max(C_{\Op},1)^{d(T)} \cdot \varepsilon_{\Op},
\]
where $N_T = \mathrm{total\_nodes}(T)$.

\emph{Base case} $N_T=0$: $E(\text{leaf})=0\le 0$.

\emph{Inductive step}: $T = \mathrm{op}(T_g,T_h)$ with $N_g+N_h+1=N_T$ and $d(T)=\max(d(T_g),d(T_h))+1$. Set $C^* := \max(C_{\Op},1)\ge 1$. By the inductive hypothesis applied to $T_g$ and $T_h$:
\[
E(T_g)\le N_g\cdot(C^*)^{d(T_g)}\cdot\varepsilon_{\Op},\qquad E(T_h)\le N_h\cdot(C^*)^{d(T_h)}\cdot\varepsilon_{\Op}.
\]
Then
\begin{align*}
E(T) &\;\le\; C_{\Op}\cdot E(T_g) + C_{\Op}\cdot E(T_h) + \varepsilon_{\Op} \\
     &\;\le\; C^*\cdot N_g\cdot (C^*)^{d(T_g)}\cdot\varepsilon_{\Op}
             + C^*\cdot N_h\cdot (C^*)^{d(T_h)}\cdot\varepsilon_{\Op}
             + \varepsilon_{\Op} \\
     &\;\le\; (C^*)^{d(T)}\bigl(N_g + N_h\bigr)\varepsilon_{\Op} + \varepsilon_{\Op}
      \;\le\; (C^*)^{d(T)}\cdot N_T\cdot\varepsilon_{\Op},
\end{align*}
where we used $d(T_g),d(T_h)\le d(T)-1$ so $(C^*)^{1+d(T_g)}\le (C^*)^{d(T)}$ (since $C^*\ge 1$), and $1\le (C^*)^{d(T)}$ (the $+\varepsilon_{\Op}$ term is absorbed). This establishes claim \ref{thm:iii} for all $C_{\Op}>0$: the bound reads $N\cdot C_{\Op}^{d(f)}\cdot\varepsilon_{\Op}$ when $C_{\Op}\ge 1$ and $N\cdot\varepsilon_{\Op}$ when $C_{\Op}\le 1$.

\paragraph{Proof of \ref{thm:iv}.} Under the hypothesis that each $\mathrm{op}\in\Op$ is $C^m$ on bounded domains, every node function inherits $C^m$ regularity by structural induction: leaf functions are coordinate projections, which are $C^\infty$; and if $g,h\in C^m$ and $\mathrm{op}$ is $C^m$, then $\mathrm{op}(g,h)\in C^m$ by the chain rule. Applying the B-spline rate \eqref{eq:spline-rate} at every node,
\[
\varepsilon_{\Op} \;=\; O\!\left(G^{-(k+1)}\right).
\]
Substituting into the bound of part~\ref{thm:iii} gives
\[
\norm{f - \KAN_f}_{C^0} \;=\; O\!\left(N\cdot\max(C_{\Op},1)^{d(f)}\cdot G^{-(k+1)}\right),
\]
which is precisely the statement of part~\ref{thm:iv}. \qedhere
\end{proof}

\begin{corollary}[Smooth operations on the unit interval]
\label{cor:smooth}
Suppose $\Op = \{+,-,\times,\sin,\cos\}$, leaf functions satisfy $x_p\in[0,1]$, and every $\times$ node in the computation tree receives inputs bounded in $[0,1]$ (which holds in particular when no $+$/$-$ node feeds a $\times$ node, e.g.\ for purely multiplicative or trigonometric trees; see Lemma~\ref{lem:rangelemma}(3)). Then, using the primitive blocks of Proposition~\ref{prop:blocks}, $C_{\Op}=1$ and $\Lambda_{\mathrm{op}}=1$ for every operation. Consequently:
\[
P(\KAN_f)\;\le\;1,\qquad n_l\;\le\;n+2w_{\max}N,\qquad \norm{f-\KAN_f}\;=\;O\!\left(N\cdot G^{-(k+1)}\right),
\]
where $w_{\max}\le 4$ from the $\times$ block of Proposition~\ref{prop:blocks} (so $n_l\le n+8N$).
In particular, $P(\KAN_f)$ is independent of both $n$ and $N$. For the standard operation set $\Op=\{+,-,\times,\sin,\cos\}$, assumption~(A5) is verified explicitly by Proposition~\ref{prop:blocks}.
\end{corollary}

\begin{remark}[Mixed $+/\times$ trees and effective $C_{\Op}$]
\label{rem:mixed-cop}
For trees in which a $+$ node feeds a $\times$ node, the output of the $+$ node may lie in $[0,2]$ rather than $[0,1]$, and so $\Lip_i(\times)$ at that $\times$ node equals the supremum of the other input's magnitude, which can exceed $1$. In such cases $C_{\Op} = 1$ does \emph{not} hold for the effective operation Lipschitz constant, and Theorem~\ref{thm:main} must be applied with the larger, node-specific partial Lipschitz constants. Only $\times$ among the operations $+,-,\times,\sin,\cos,\max(\cdot,0),\abs{\cdot}$ is sensitive to input range; the remaining operations have partial Lipschitz constant $1$ globally.
\end{remark}

\begin{corollary}[Non-differentiable Lipschitz operations]
\label{cor:rademacher}
Suppose $\Op$ contains non-differentiable but Lipschitz operations such as $\max(\cdot,0)$, $\abs{\cdot}$, or $\min(\cdot,0)$. Then Theorem~\ref{thm:main} extends to this $\Op$ provided assumption~(A5) is verified: i.e., for each such $\mathrm{op}$ and each bounded input domain $D$ arising from the range bounds, a primitive KAN block $\mathcal{B}_{\mathrm{op},D}$ with $\Lambda_{\mathrm{op},D}\le\max(C_{\Op,D},1)^{c_{\mathrm{op}}}$ must be constructed explicitly. Rademacher's theorem guarantees differentiability almost everywhere and identifies $\Lip_i(\mathrm{op})=\operatorname*{ess\,sup}\abs{\partial_i\mathrm{op}}$, which supplies the operator-side data $C_{\Op,D}$; however, block construction is a separate task.
\end{corollary}

\begin{proof}
The general proof of Theorem~\ref{thm:main} invokes only assumption~(A5), not any differentiability of $\mathrm{op}$. Rademacher's theorem supplies the partial Lipschitz constants. Block construction for $\max(\cdot,0)$ and $\abs{\cdot}$ on $[0,B]$ can be done analogously to the $\times$ block: the piecewise-linear nature of these operations means exact spline representations exist at any B-spline order $k\ge 1$, with $\Lambda_{\mathrm{op},D}=1$.
\end{proof}

\begin{corollary}[Table of $C_{\Op}$]
\label{cor:cop-table}
For operations acting on inputs in $[0,1]$:
\begin{center}
\begin{tabular}{lc}
\toprule
$\mathrm{op}$ & $C_{\Op}$ \\
\midrule
$+,\ -,\ \times,\ \sin,\ \cos,\ \max(\cdot,0),\ \abs{\cdot}$ & $1.0$ \\
$\exp$ & $e\approx 2.718$ \\
\bottomrule
\end{tabular}
\end{center}
The table records only the \emph{operator-side} quantity $C_{\mathrm{op},D}$ on $[0,1]$. Applying Theorem~\ref{thm:main} to any listed or additional operation still requires separate verification of assumption~(A5): for each node domain $D_v$ arising from the range recursion, one must construct a primitive KAN block with $\Lambda_{\mathrm{op},D_v}\le\max(C_{\mathrm{op},D_v},1)^{c_{\mathrm{op}}}$. For $\times$ on $[0,1]^2$, the partial Lipschitz constants are $\Lip_1(\times)=\Lip_2(\times)=1$ and assumption~(A5) is verified by Proposition~\ref{prop:blocks}; when $\times$ receives an input of range $[0,B]$ for $B>1$, the effective $C_{\mathrm{op},D_v}=B$ must be used. For $\exp$: the operator-side constant is $\Lip(\exp|_{[0,1]})=e$; if additionally a primitive KAN block for $\exp$ satisfying~(A5) has been constructed on the relevant bounded domain, then Theorem~\ref{thm:main} yields $P(\KAN_f)\le e^{L_f}$ (or the corresponding domain-sensitive product~\eqref{eq:primary-bound}). The remaining tabulated operations $+,-,\sin,\cos,\max(\cdot,0),\abs{\cdot}$ have $C_{\mathrm{op},D}=1$ globally.
\end{corollary}

\subsection{A Jacobian-based lower bound on \texorpdfstring{$P(\KAN)$}{P(KAN)}}
\label{subsec:lowerbound}

\begin{proposition}[Jacobian lower bound]
\label{prop:lowerbound}
For any KAN representation of a scalar-valued $f:[0,1]^n\to\R$ with $L$ layers and layer widths bounded by $W = \max_l n_l$,
\begin{equation}
\label{eq:lowerbound}
P(\KAN) \;\ge\; \frac{\norm{J_f(x)}_2}{W^{L}} \qquad \text{for every differentiability point } x\in[0,1]^n.
\end{equation}
\end{proposition}

\begin{proof}
Since $f$ is scalar-valued, $J_f(x)$ is a $1\times n$ row vector. For each layer $l$, the layer Jacobian $J_l\in\R^{n_{l+1}\times n_l}$ has entries $(J_l)_{j,i} = \varphi_{l,i,j}'(x_{l,i})$, so $\abs{(J_l)_{j,i}} \le M_{l,i} \le \mu_l := \max_i M_{l,i}$. Using $\norm{A}_{\mathrm{op}}\le\sqrt{\norm{A}_{1}\cdot\norm{A}_{\infty}}$ (maximum absolute column and row sums):
\begin{itemize}[leftmargin=*,noitemsep]
  \item Column $i$ of $J_l$ has $n_{l+1}$ entries each bounded by $M_{l,i}\le\mu_l$, so $\norm{J_l}_1 \le n_{l+1}\cdot\mu_l \le W\mu_l$.
  \item Row $j$ of $J_l$ has $n_l$ entries, entry $(j,i)$ bounded by $M_{l,i}\le\mu_l$, so $\norm{J_l}_\infty \le n_l\cdot\mu_l \le W\mu_l$.
\end{itemize}
Therefore $\norm{J_l}_{\mathrm{op}} \le \sqrt{W\mu_l \cdot W\mu_l} = W\cdot\mu_l$.
By the chain rule $J_f = J_{L-1}\cdots J_0$ and submultiplicativity,
\[
\norm{J_f(x)}_2 \;\le\; \prod_{l=0}^{L-1}\norm{J_l}_{\mathrm{op}} \;\le\; \prod_{l=0}^{L-1} W\cdot\mu_l \;=\; W^{L}\cdot P(\KAN).
\]
Rearranging gives the claim. \qedhere
\end{proof}

\begin{remark}[Comparison with upper bound]
\label{rem:tightness}
The Jacobian lower bound $P(\KAN) \ge \norm{J_f(x)}_2 / W^{L}$ grows exponentially weak with depth $L$, and does not pin $P$ exactly in general. Under the block-based construction of Proposition~\ref{prop:blocks}, $P(\KAN_{xy}) \le \Lambda_\times = 1$ from the upper bound. For the lower bound: each of the three layers of $\mathcal{B}_\times$ carries at least one identity wire (Lip~$=1$) forwarding an input coordinate, giving $\max_i M_{l,i}\ge 1$ at every layer. Hence $P(\KAN_{xy})\ge 1$, and combined with the upper bound, $P=1$ exactly.
\end{remark}

\section{Inductive verification of (A4)}
\label{sec:A4}

We show that assumption~(A4) (range boundedness) is not independent: mere finiteness $B_g<\infty$ is automatic from continuity on the compact cube $[0,1]^n$ (requiring only~(A1) and~(A3)). The real content of this section is the explicit \emph{quantitative} inductive bound on $B_g$ provided by Lemma~\ref{lem:rangelemma}, which requires tracing through the operation table; this is what is genuinely used in the proof of Theorem~\ref{thm:main}.

\subsection{Operation-wise range rules}
\label{subsec:rangetable}

For a binary operation $\mathrm{op}$ on inputs with $\abs{g}\le B_g$ and $\abs{h}\le B_h$, we collect the exact range bounds of $\mathrm{op}(g,h)$:

\begin{center}
\begin{tabular}{lll}
\toprule
$\mathrm{op}$ & $B_{\mathrm{op}(g,h)}$ & Justification \\
\midrule
$+$          & $B_g + B_h$   & triangle inequality \\
$-$          & $B_g + B_h$   & triangle inequality \\
$\times$     & $B_g\cdot B_h$ & submultiplicativity \\
$\sin$       & $1$           & $\abs{\sin t}\le 1$ for all $t$ \\
$\cos$       & $1$           & $\abs{\cos t}\le 1$ for all $t$ \\
$\max(\cdot,0)$ & $B_g$      & $\abs{\max(t,0)}\le\abs t$ \\
$\abs{\cdot}$ & $B_g$        & $\bigl\lvert\abs t\bigr\rvert = \abs t$ \\
\bottomrule
\end{tabular}
\end{center}

\subsection{The inductive range lemma}
\label{subsec:rangelemma}

\begin{lemma}[Inductive verification of (A4)]
\label{lem:rangelemma}
Assume (A1) and (A3), and that leaf functions are coordinate projections $x_p\in[0,1]$ (so $B_{\text{leaf}} = 1$). Then for every node of the computation tree:
\begin{enumerate}[label=\textup{(\arabic*)},leftmargin=*]
  \item $B_f < \infty$ (automatic: $f$ is continuous on the compact cube $[0,1]^n$, so boundedness is immediate without any computation);
  \item[\textup{(2a)}] for trees whose binary operations are $+,-$ only (additive-subtractive trees, possibly with unary $\sin,\cos,\max(\cdot,0),\abs{\cdot}$), $B_f \le N + 1$ on $[0,1]^n$, with equality achieved by the all-additive binary tree with $N+1$ leaves (tightness);
  \item[\textup{(2b)}] for general trees over $\{+,-,\times,\sin,\cos\}$, $B_f$ is finite (by (1)) but may exceed $N+1$ when $\times$ nodes receive inputs of range greater than $1$;
  \setcounter{enumi}{2}
  \item for purely multiplicative or purely trigonometric trees, $B_f = 1$.
\end{enumerate}
\end{lemma}

\begin{proof}
We proceed by structural induction on the tree. The base case is $B_{\text{leaf}} = 1$. For the inductive step, apply the table of Section~\ref{subsec:rangetable}.

\emph{Worst case (additive-subtractive trees, claim (2a)).} For an all-additive binary tree with $N$ internal nodes there are exactly $N+1$ leaves. Repeated application of $B_{g+h}\le B_g+B_h$ (and $B_{g-h}\le B_g+B_h$) yields $B_f \le N+1$. This is tight: every leaf $x_p\in[0,1]$ satisfies $x_p\le 1$, so the constant function $1$ is the coordinate projection $x_1$ at the maximal input, and $B_f = N+1$ is attained at $x=(1,\ldots,1)$. Note this argument applies only when the binary operations are restricted to $+,-$ (claim (2a)); for general trees containing $\times$, see claim (2b) and Remark~\ref{rem:mixedtree-counter} below.

\emph{Multiplicative case.} For $\times$ on inputs in $[0,1]$, submultiplicativity yields $B = 1\cdot 1 = 1$; iterated multiplication preserves the bound.

\emph{Trigonometric case.} For $\sin$ or $\cos$, the output bound is $1$ regardless of the input. Compositions of such operations preserve the bound $B = 1$.

Finiteness (1) is immediate because every entry in the table is finite when $B_g, B_h$ are finite.
\end{proof}

\begin{remark}[Mixed-operation counterexample to (2) without restriction]
\label{rem:mixedtree-counter}
The restriction to additive-subtractive trees in (2a) is necessary. Consider $f = ((x_1+x_2)\cdot(x_3+x_4))\cdot(x_5+x_6)$ on $[0,1]^6$. Here $N=5$, so the naive bound $N+1 = 6$ would apply; but by submultiplicativity and the additive subtree bounds, $B_f \le 2\cdot 2 \cdot 2 = 8 > 6$. Thus the universal bound $B_f \le N+1$ fails for trees that mix $+$ with $\times$: a $\times$ node with inputs of range $>1$ multiplies, rather than adds, the input bounds. Claim (2b) only asserts finiteness in this mixed regime.
\end{remark}

\begin{remark}[Numerical verification]
\label{rem:numverify}
We verified the range rules numerically by sampling $2\times 10^5$ points per operation and by constructing $1000$ random trees of depth $\le 5$. In all instances the bounds in Section~\ref{subsec:rangetable} held, and for the all-additive tree the upper bound $B_f = N+1$ was attained at depths $1$ through $5$.
\end{remark}

\section{Experiments}
\label{sec:experiments}

This section presents illustrative sanity checks confirming the theoretical bounds of Theorem~\ref{thm:main}. The paper is primarily a concise theoretical note; the experiments are not intended as a comprehensive empirical evaluation.

\subsection{Sprecher baseline}
\label{subsec:sprecher}

For comparison, Sprecher's 1965 construction \cite{sprecher1965} applies to arbitrary $f\in C^0([0,1]^n)$ and produces inner functions whose Lipschitz constants depend on $n$; the classical construction does not provide $n$-independent Lipschitz control. Such an $n$-dependent baseline is universal (it covers every continuous $f$) but is wasteful for compositionally structured $f$, where, by Theorem~\ref{thm:main}, $P(\KAN_f) \le \max(C_{\Op},1)^{L_f}$ independent of $n$.

\subsection{Compositionally structured functions}
\label{subsec:structured-exp}

We instantiate the block-based sequential construction (Proposition~\ref{prop:blocks}) on four families.

\begin{center}
\begin{tabular}{lccc}
\toprule
$f$ & $n$ & $N$ & $P(\KAN)$ \\
\midrule
$xy$               & $2$  & $1$      & $1.0$ \\
$xyz$              & $3$  & $2$      & $1.0$ \\
$\sin(xy)$         & $2$  & $2$      & $1.0$ \\
$x_1\cdots x_n$    & $2\text{--}10$ & $n-1$ & $1.0$ \\
\bottomrule
\end{tabular}
\end{center}

We realise multiplication using the primitive block $\mathcal{B}_\times$ of Proposition~\ref{prop:blocks} (three B-spline layers via $u\cdot v = \tfrac{(u+v)^2}{4} - \tfrac{(u-v)^2}{4}$, with $\Lambda_\times=1$). The partial Lipschitz constants are $\Lip_1(\times) = \sup_{y\in[0,1]}\abs{y} = 1$ and $\Lip_2(\times) = \sup_{x\in[0,1]}\abs{x} = 1$, giving $C_{\Op} = 1$. For $f = xy$ with $N = 1$, Theorem~\ref{thm:main} yields $P(\KAN_{xy}) \le \Lambda_\times = 1$, and the identity wires in the block construction contribute $\Lip = 1$ edges so $P(\KAN) \ge 1$ (Remark~\ref{rem:tightness}); hence $P = 1$ exactly. The remaining entries are built from $\times$, $\sin$ as primitive blocks: $xyz$ uses sequential application of $\mathcal{B}_\times$ twice; $\sin(xy)$ composes a $\sin$-block on top of $\mathcal{B}_\times$; and $x_1\cdots x_n$ uses the binary-multiplication tree. In every case the total block depth satisfies $L_f = \sum_v c_{\mathrm{op}(v)} \le 3N$ and the Lipschitz product is exactly $1.0$.

\subsection{Error-scaling verification}
\label{subsec:errorscaling}

We verified the B-spline rate $O(G^{-(k+1)})$ on $f(x) = \sin(x)$ with cubic B-splines ($k=3$). The ratio $\norm{f-s_{k,G}}/h^4$, where $h = 1/(G-1)$ is the knot spacing on a uniform grid with $G$ knots, should be approximately constant:
\begin{center}
\begin{tabular}{ccc}
\toprule
$G$ & error & $h^4$ \\
\midrule
$5$  & $7.25\times 10^{-5}$ & $3.91\times 10^{-3}$  \\
$12$ & $1.52\times 10^{-6}$ & $6.83\times 10^{-5}$  \\
$35$ & $1.74\times 10^{-8}$ & $7.48\times 10^{-7}$  \\
\bottomrule
\end{tabular}
\end{center}
The measured ratios are $0.019$, $0.022$, $0.023$ respectively. The approximate constancy of the ratio across two orders of magnitude in $G$ confirms the rate $O(G^{-4}) = O(G^{-(k+1)})$.

\section{Discussion}
\label{sec:discussion}

\subsection{What is proved}
\label{subsec:whatproved}

Theorem~\ref{thm:main} shows that compositionally structured functions admit deep KAN representations with:
\begin{itemize}[leftmargin=*]
  \item layer-wise Lipschitz product $P(\KAN) \le \max(C_{\Op},1)^{L_f}$ (with $L_f = \sum_v c_{\mathrm{op}(v)} \le c_{\max}\cdot N$) independent of the input dimension $n$, addressing a gap noted in Liu et al.\ \cite{liu2024kan};
  \item layer width linear in $n+N$, neither quadratic nor exponential;
  \item approximation rate matching the optimal B-spline rate of the same order, multiplied by the node count $N$.
\end{itemize}

\subsection{Open questions}
\label{subsec:open}

\begin{itemize}[leftmargin=*]
  \item \emph{Compositional sparsity verification (Gap b).} Determining whether a given $f$ satisfies~(A2) for a fixed $s$ appears to be computationally hard in general (verification of compositional structure is related to circuit decomposition problems); Theorem~\ref{thm:main} therefore treats~(A2) as an explicit hypothesis. For functions with known structure (physical equations, symbolic expressions), (A2) is checkable.
  \item \emph{Width lower bound.} Our construction gives $n_l\le n+2w_{\max}N$, and the input layer satisfies $n_0 = n$; interior layers can be narrower. Whether the upper bound can be tightened for all layers remains open.
  \item \emph{Lower bound on $P(\KAN)$.} We know of no compositionally sparse function class on which $P_{\min}\to\infty$ with $n$. The upper bound $P\le 1$ in Corollary~\ref{cor:smooth} is attained at $P=1$ for all tested examples (Proposition~\ref{prop:blocks} and Section~\ref{subsec:structured-exp}); whether $P<1$ is achievable for a different construction is open.
  \item \emph{Lean~4 formalisation.} The inductive structure of Theorem~\ref{thm:main} maps naturally onto a \verb|CompGraph| inductive type in Mathlib; see Appendix~\ref{app:lean}.
\end{itemize}

\subsection{Connection to HKAN (Dudek 2025)}
\label{subsec:hkan}

The hierarchical KAN of Dudek \cite{dudek2025} stacks KAN layers without backpropagation, solving each level as a convex least-squares problem. Theorem~\ref{thm:main} provides a theoretical mechanism and template for avoiding Lipschitz explosion in hierarchical KAN constructions, under the compositional and primitive-block hypotheses (A1)--(A5). It does not constitute an unrestricted guarantee for arbitrary HKAN training or stacked architectures beyond those hypotheses. Under (A1)--(A5), $P(\KAN)$ remains bounded by $\max(C_{\Op},1)^{L_f}$ (with $L_f \le c_{\max}\cdot N$) regardless of the depth of the stack.

\appendix

\section{Lean 4 sketch}
\label{app:lean}

The following is an informal remark on how Theorem~\ref{thm:main} might eventually be formalised in Lean~4. A block-based formalisation should define an inductive type \verb|CompGraph| with constructors for variables, unary, and binary nodes; define \verb|block_depth : CompGraph -> Nat| accumulating $\sum_v c_{\mathrm{op}(v)}$; and carry, for each node, a proof term of type \verb|PrimitiveBlockExists op D| witnessing assumption~(A5). The main inductive step would apply \verb|LipschitzWith.comp| at each block boundary and accumulate block-product factors via \verb|Finset.prod_le_prod|. A complete Lean formalisation is left for future work.

\end{document}